\documentclass{article}

\usepackage{arxiv}

\usepackage[utf8]{inputenc} 
\usepackage[T1]{fontenc}    
\usepackage{hyperref}       
\usepackage{url}            
\usepackage{booktabs}       
\usepackage{amsfonts}       
\usepackage{nicefrac}       
\usepackage{microtype}      
\usepackage{graphicx}
\usepackage{natbib}
\usepackage{doi}

\usepackage{amsmath}
\usepackage{amssymb}
\usepackage{bbm}
\usepackage{amsthm}
\usepackage{enumitem} 
\usepackage[nameinlink, noabbrev, capitalize]{cleveref} 
\usepackage{multirow}

\newtheorem{theorem}{Theorem}[section] 

\title{A new perspective on classification: optimally allocating limited resources to uncertain tasks}

\author{Toon Vanderschueren\thanks{Corresponding author} \\
	KU Leuven, University of Antwerp\\
	\href{mailto:toon.vanderschueren@kuleuven.be}{\texttt{toon.vanderschueren@kuleuven.be}} \\
	\And
	Bart Baesens \\
	KU Leuven, University of Southampton \\
	\href{mailto:bart.baesens@kuleuven.be}{\texttt{bart.baesens@kuleuven.be}} \\
	\AND 
	Tim Verdonck \\
	University of Antwerp \\
	\href{mailto:tim.verdonck@uantwerpen.be}{\texttt{tim.verdonck@uantwerpen.be}} \\
	\And
	Wouter Verbeke \\
	KU Leuven \\
	\href{mailto:wouter.verbeke@kuleuven.be}{\texttt{wouter.verbeke@kuleuven.be}} \\
}

\date{}

\renewcommand{\shorttitle}{\textit{arXiv} Template}

\hypersetup{
pdftitle={A new perspective on classification: optimally allocating limited resources to uncertain tasks},
pdfsubject={},
pdfauthor={Toon~Vanderschueren},
pdfkeywords={Optimal resource allocation, Classification, Learning to rank}, 
}

\usepackage{dsfont}
\usepackage{bbm}

\usepackage{hyperref}
\usepackage[nameinlink, noabbrev, capitalize]{cleveref} 

\usepackage{multirow}
\usepackage[table]{xcolor}
\usepackage[style=base, font=footnotesize]{caption}
\usepackage{subcaption} 
\usepackage{array}
\usepackage{fdsymbol} 
\usepackage{diagbox} 
\usepackage{booktabs}

\newcolumntype{H}{>{\setbox0=\hbox\bgroup}c<{\egroup}@{}}

\newcommand{\PreserveBackslash}[1]{\let\temp=\\#1\let\\=\temp}
\newcolumntype{C}[1]{>{\PreserveBackslash\centering}p{#1}}
\newcolumntype{R}[1]{>{\PreserveBackslash\raggedleft}p{#1}}
\newcolumntype{L}[1]{>{\PreserveBackslash\raggedright}p{#1}}

\allowdisplaybreaks 

\begin{document}
\maketitle

\begin{abstract}
A central problem in business concerns the optimal allocation of limited resources to a set of available tasks, where the payoff of these tasks is inherently uncertain. In credit card fraud detection, for instance, a bank can only assign a small subset of transactions to their fraud investigations team. Typically, such problems are solved using a classification framework, where the focus is on predicting task outcomes given a set of characteristics. Resources are then allocated to the tasks that are predicted to be the most likely to succeed.  However, we argue that using classification to address task uncertainty is inherently suboptimal as it does not take into account the available capacity. Therefore, we first frame the problem as a type of assignment problem. Then, we present a novel solution using learning to rank by directly optimizing the assignment's expected profit given limited, stochastic capacity. This is achieved by optimizing a specific instance of the net discounted cumulative gain, a commonly used class of metrics in learning to rank. Empirically, we demonstrate that our new method achieves higher expected profit and expected precision compared to a classification approach for a wide variety of application areas and data sets. This illustrates the benefit of an integrated approach and of explicitly considering the available resources when learning a predictive model.
\end{abstract}


\section{Introduction}

Optimally allocating limited resources is a central problem in economics \citep{samuelson2010economics} and operations research \citep{ward1957optimal, everett1963generalized}. It is often complicated further by uncertainty inherent to the considered problem. On the one hand, future resource capacity may be limited and not known exactly in advance. On the other hand, the tasks that require resources might have uncertain payoff. This situation is commonly encountered in various real-world applications. For example, in credit card fraud detection, fraud analysts can only investigate a limited number of transactions each day. Similarly, in direct marketing, a company may only be able to target a subset of customers in a marketing campaign. The challenge is how to optimally allocate resources to maximize business pay-off, e.g., how to optimally allocate fraud investigators to suspicious transactions to minimize losses due to fraud. By learning from historical data, machine learning models can assist decision-makers by predicting the most relevant tasks based on their characteristics.

Prior work addresses the problem of uncertain task outcomes via classification. The most promising tasks can be identified by estimating the probability of success for each task. The problem of stochastic, limited capacity can then be addressed separately in a second stage, when assignment decisions are made by prioritizing tasks based on the estimated probabilities to result in a successful outcome. In this article, however, we argue and demonstrate that this approach based on classification models is suboptimal when resources are limited because a classification model does not take capacity limitations into account. Hence, although only the most promising tasks can be executed, the model focuses equally on accurately predicting probabilities for tasks that are highly unlikely to be successful and, consequently, to be executed. 

Therefore, we propose a novel approach based on learning to rank that simultaneously accounts for both resource and task uncertainty.
When resources are limited, we demonstrate that this approach is superior to allocation based on classification. First, we show how learning to rank can directly optimize the assignment's expected profit given limited, stochastic capacity. By considering the available capacity during optimization, the model focuses on correctly ranking the most promising tasks, proportional to their likelihood of being processed under limited capacity. Second, while instances are processed individually in classification, learning to rank explicitly considers a task's relevance in comparison to the other available tasks. The benefit of this approach is that we only care about relative positions in the ranking, corresponding to the need to prioritize tasks relative to each other. 

Our contributions are threefold. First, we formalize the problem of allocating limited, stochastic resources to uncertain tasks by framing it as an assignment problem. Second, we propose a novel, integrated predict-and-optimize approach to solve this problem based on learning to rank. We contrast our approach with a two-stage predict-then-optimize framework that first uses a classification model to predict task outcomes and then solves the assignment problem using the predicted task probabilities. Third, we compare both methods empirically using various real life data sets from different application areas.

\section{Problem formulation}
\label{sec:problem formulation}

\begin{figure}[t]
    \includegraphics[width=\textwidth, trim={0 0.2cm 0 0}]{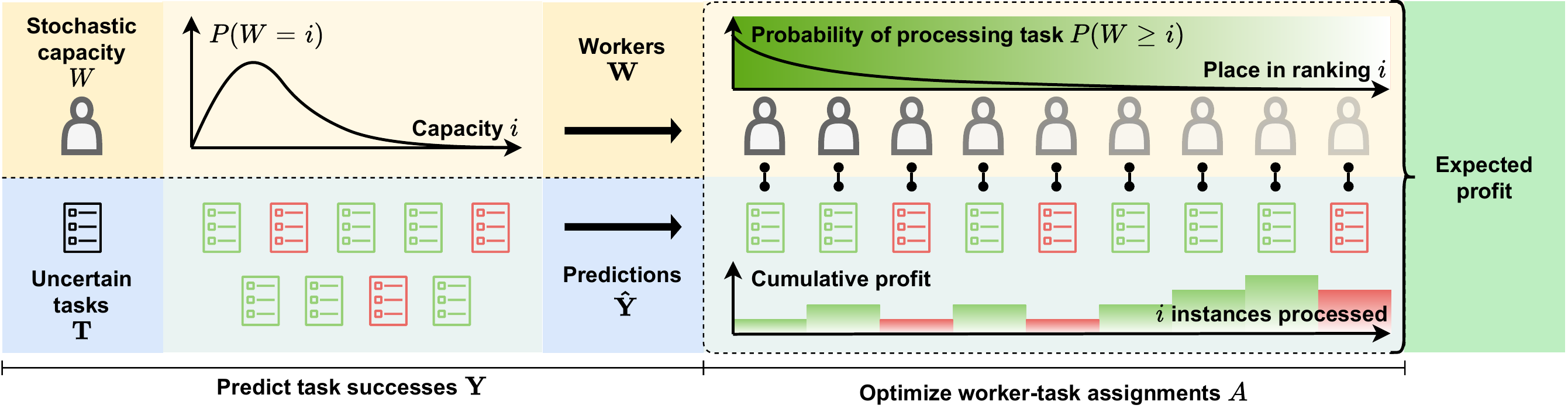}
    \caption{
    \textbf{Problem overview.} \normalfont{We formulate our setting as a type of linear assignment problem where two sources of uncertainty must be considered: stochastic worker capacity and uncertain task outcomes. To account for stochastic capacity in the assignment problem, the capacity distribution is converted to workers with decreasing processing probabilities. Task outcomes are also uncertain and need to be predicted. The ultimate objective is to assign workers to tasks to maximize the resulting expected profit.}
    }
    \label{fig:problem_overview}
\end{figure}

We address the problem of optimally assigning limited and stochastic resources to tasks with uncertain outcomes to maximize the expected profit. We formalize it as a linear assignment problem where both workers and tasks are sources of uncertainty. The exact number of workers is uncertain at the time when resources need to be allocated, but we assume it is governed by a known probability distribution. In practice, this distribution can be estimated from historical data on the available resources or based on domain knowledge. Alternatively, a deterministic capacity can be considered. Second, task outcomes are also uncertain and need to be predicted using historical data on similar tasks. A graphical overview of the problem is shown in \cref{fig:problem_overview}. In the following, we introduce and formally define each element of the assignment problem.

\paragraph{Stochastic capacity.}
The available resources or number of workers $W$ is a discrete random variable described by a known probability distribution. In this work, we consider a common situation where the expected capacity $\mathbb{E}(W)$ is smaller than the number of available tasks $N$. The stochastic capacity can be converted to a sequence of $N$ workers with monotonically decreasing success rates. This rate $w_i$ equals the worker's probability of being available given $W$ and is described by the complementary cumulative probability distribution function: $w_i = P(W \geq i) = 1 - F_W(i)$. This yields a monotonically decreasing sequence of $N$ worker success rates $\mathbf{W} = \begin{pmatrix} w_1 & \dots & w_N \end{pmatrix} = \{1-F_W(i)\}_{i=1}^N$ with $w_1 \geq \dotso \geq w_N$.

\paragraph{Uncertain tasks.}
There is also uncertainty regarding task outcomes. To address this uncertainty, we predict it using historical data on similar tasks. Let $\mathcal{T} = (\mathcal{X}, \mathcal{Y}, \mathcal{V})$ be the domain of all possible tasks $t_i = (\mathbf{x}_i, y_i, \mathbf{v}_i)$, where $\mathbf{x}_i \in \mathcal{X} \subset \mathbb{R}^d$ is a set of characteristics and $y_i \in \mathcal{Y} = \{0, 1\}$ is a binary label equal to 1 if the task is successful and 0 otherwise. Moreover, $\mathbf{v}_i = \{v^+_i, v^-_i\} \in \mathcal{V} \subset \mathbb{R}^2$ denotes the payoff if the task is executed, with $v^+_i$ if task $i$ was successful ($y_i = 1$) and $v^-_i$ otherwise. A task's reward is defined as $r_i = y_i v^+_i + (1-y_i)v^-_i$. We have $N$ available tasks to be allocated $\mathbf{T} = \{(\mathbf{x}_i, y_i, \mathbf{v}_i): i=1,\dots,N \}$, although $y_i$ is unknown when resources need to be allocated. Given historical data, a predictive model can estimate task outcomes $y_i$ resulting in $N$ predictions.

\paragraph{Matching workers and tasks.}
Workers and tasks can then be combined in an expected profit matrix $P = \begin{pmatrix} p_{ij} \end{pmatrix}$, where $p_{ij} = w_i r_j$ is the profit of assigning worker $i$ to task $j$ for $i,j=1,\dots,N$. Given $P$, the goal is to find the optimal assignment matrix $A = \begin{pmatrix} a_{ij} \end{pmatrix}$, where $a_{ij} = 1$ if worker $i$ is assigned to task $j$ and 0 otherwise, for $i,j=1,\dots,N$. This results in the following balanced linear assignment problem:
\begin{alignat}{3}
    \text{maximize} & \sum_{i=1}^N \sum_{j=1}^N p_{ij} a_{ij} & \\
    \text{subject to}
        & \sum_{i=1}^N a_{ij} = 1  & j = 1, \dots, N; \label{eq:cond_task} \\ 
        & \sum_{j=1}^N a_{ij} = 1  & i = 1, \dots, N; \label{eq:cond_worker} \\ 
        & a_{ij} \in \{0,1\}       & i,j = 1, \dots, N;  \label{eq:integer}
\end{alignat}
where conditions \ref{eq:cond_task} and \ref{eq:cond_worker} specify that each task is assigned to exactly one worker and vice versa; condition \ref{eq:integer} imposes absolute assignments by restricting $a_{ij}$ to 0 or 1.

\section{Related work}

The proposed solution in this paper relates to prior work on uncertainty in assignment problems, predict-and-optimize, classification and learning to rank. In this section, we briefly introduce each line of work and its relationship to our contribution.

\subsection{Uncertainty in assignment problems}
Optimal allocation of resources and decision-making under uncertainty are key problems in operations research \citep{ward1957optimal, everett1963generalized}. In this work, we consider an assignment problem. This is a general problem formulation in which the goal is to find an optimal matching of workers and tasks subject to certain constraints. This type of problem has been analyzed extensively \citep{burkard2012assignment} and applied to a diverse range of tasks \citep[e.g.,][]{alonso2017demand, bertsimas2019optimizing}. Moreover, various extensions consider different sources of uncertainty: uncertain worker capacity, uncertain task presence (i.e., outcomes), or uncertain task-worker profits \citep{toktas2006addressing, krokhmal2009random}. This work focuses on a specific type of a linear assignment problem, in which we simultaneously address two sources of uncertainty: uncertain capacity and uncertain task success. However, instead of assuming that task success follows a probability distribution, we use a predictive model to estimate it.

\subsection{Predict-and-optimize}
The intersection of operations research and machine learning has increasingly drawn the attention of researchers from both fields \citep{lodi2017learning, bengio2021machine}. In particular, recent work on predict-and-optimize is relevant \citep{donti2017taskbased, wilder2019melding, elmachtoub2021smart}. The central aim in predict-and-optimize is to align a predictive model more closely with the downstream decision-making context \citep{mandi2020smart}. This is achieved by fusing the prediction and optimization phases and training the model in an end-to-end manner, with the aim of obtaining higher quality decisions \citep{kotary2021end}. Ranking specifically has been studied in this context: Demirovi\'{c} et al. use ranking methods for the knapsack problem \citep{demirovic2019investigation} and study ranking objectives of combinatorial problems in general \citep{demirovic2019predict+}, though both are limited to linear models. In contrast, our proposed approach is compatible with any type of learner that can be trained using gradient-based optimization, but it is applicable only to our specific, though commonly encountered, formulation of the assignment problem.

\subsection{Classification}
Classification is a task in machine learning where the goal is to predict the class of an instance given its characteristics. For instance, classifying a task as either successful or not is a binary classification problem. Existing work typically considers the applications in this paper as classification problems, e.g., fraud detection \citep{vanvlasselaer2017gotcha, cerioli2019newcomb}, credit scoring \citep{baesens2003benchmarking, lessmann2015benchmarking}, direct marketing \citep{baesens2002bayesian} and customer churn prediction \citep{verbeke2011building, verbeke2012new_churn}. Moreover,  to align the models more closely with the decision-making context, cost-sensitive classification has been used \citep{bahnsen2014examplelogistic, petrides2020profit_credit, hoppner2020profit, hoppner2022instance}. Cost-sensitive methodologies incorporate the costs of different decisions into the optimization or use of predictive models \citep{elkan2001foundations, petrides2021csensemble}. Cost-sensitive variants have been proposed for different classification models, such as logistic regression and gradient boosting \citep{bahnsen2014examplelogistic, hoppner2022instance}. The output of a classification model is often used to rank instances, reflected by widely used evaluation metrics that analyze this ranking, such as the receiver operating characteristics curve and precision--recall curve \citep{davis2006relationship}. However, in contrast to our work, these approaches do not consider the available capacity during optimization of the models. Although limited capacity has been acknowledged in the literature (e.g., in fraud detection \citep{dal2017credit}, direct marketing \citep{bose2009quantitative} or churn prediction \citep{hadden2007computer}), no existing solution explicitly addresses this issue.

\subsection{Learning to rank}
In learning to rank, the goal is to predict the order of instances relative to each other, based on their characteristics.
Although learning to rank originated in the field of information retrieval, it is a general framework that has been applied to a variety of problems that have traditionally been solved with classification models, such as software defect prediction \citep{yang2014learning}, credit scoring \citep{coenen2020machine} and uplift modeling \citep{devriendt2020learning}. Moreover, similar to cost-sensitive classification, the learning to rank framework has been extended to incorporate costs of instances to align the optimization of the model more closely with the resulting decisions \citep{mcbride2019cost}.
However, our approach is the first to explicitly consider the available capacity during the optimization of the ranking model.
{\parfillskip=0pt \emergencystretch=.5\textwidth \par}

\section{Methodology}

We present two approaches for the problem presented in \cref{sec:problem formulation}. On the one hand, a two-stage predict-then-optimize framework can be used. In the first stage, we predict the task successes $\mathbf{\hat{Y}}$. Here, we show how different types of classification objectives can be used to predict task success. In the second stage, we optimize the assignment of tasks to workers to obtain an assignment matrix $A$. For this, we provide an analytical solution and prove its optimality. On the other hand, we present an integrated predict-and-optimize framework for prediction and optimization by leveraging learning to rank techniques.

\subsection{Two-stage predict-then-optimize}

This section presents a conventional two-stage approach for solving the problem. In the first stage, a classification model predicts each task's probability of success. Existing approaches in classification \citep{murphy2012machine, hoppner2022instance} can be used to optimize this model for either accuracy or profit. In the second stage, tasks are assigned to workers based on these predicted probabilities. We present a straightforward procedure for this assignment and prove its optimality.

\subsubsection{Predicting task outcomes using classification.}

To handle the uncertainty regarding task outcomes, we train a classification model to predict whether a task will be successful. Given historical data $\mathcal{D}_\text{Train}$, the goal is to predict $y_i$ using a classifier $f_\theta: \mathcal{X} \to [0, 1]: \mathbf{x} \mapsto f_\theta(\mathbf{x})$ defined by parameters $\theta \in \Theta$ that predicts the probability of a task being successful. Classifier training can be accomplished with different objective functions. We present two alternatives: one that focuses optimization on accuracy and one that optimizes the classification cost.

The conventional approach is to train the classifier with the aim of maximizing accuracy. This can be achieved using the maximum likelihood approach or, equivalently, by minimizing the cross-entropy loss function \citep{murphy2012machine}:
\begin{equation}
    \mathcal{L}^\text{CE} = 
        y_i \text{log } f_\theta(\mathbf{x}_i) + (1-y_i)\text{log}\big(1-f_\theta(\mathbf{x}_i)\big).
\label{eq:ce}
\end{equation}

A drawback of this approach is that the solution ignores some of the problem specifications. Some tasks are more important to classify correctly than others, depending on their cost (or profit) when executed. Therefore, in cost-sensitive learning, these costs are incorporated into the training of a model. In classification, the cost of a decision depends on whether it was classified correctly and on the task itself. These costs are formalized with the concept of a cost matrix $\mathbf{c}_i$ \citep{elkan2001foundations}:
\begin{equation}
    \begin{matrix}
     & \text{\bf\small Actual class } $$y_i$$ \\[0.1em]
     & \begin{matrix}
     \hspace{2pt} 0 \hspace{10pt} & \hspace{10pt} 1
     \end{matrix} \\[0.2em]
     
    \begin{matrix}
      \multirow{2}{*}{\text{\bf\small Predicted class } $\hat{y}_i$} & 0 \\[0.3em]
       & 1 \\[0.3em]
    \end{matrix}
     &
    \begin{pmatrix}
    c^\text{TN}_i & c^\text{FN}_i \\[0.3em]
    c^\text{FP}_i & c^\text{TP}_i \\[0.3em]
    \end{pmatrix} \\
    \end{matrix}
\label{eq:cost_matrix}
\end{equation}
This way, we can directly minimize the average expected cost of predictions, as an alternative to the cross-entropy loss \citep{bahnsen2014examplelogistic, hoppner2022instance}:
\begin{align}
\begin{split}
    \mathcal{L}^\text{AEC} &= 
         y_i \Big( f_\theta(\mathbf{x}_i) c_i^\text{TP} + \big(1 - f_\theta(\mathbf{x}_i)\big) c_i^\text{FN} \Big)
         \\
         &+ (1-y_i) \Big( f_\theta(\mathbf{x}_i) c_i^\text{FP} + \big(1 - f_\theta(\mathbf{x}_i)\big) c_i^\text{TN} \Big).
\label{eq:aec}
\end{split}
\end{align}
$\mathcal{L}^\text{AEC}$ is a semidirect predict-and-optimize method: it incorporates some information of the downstream decision-making task, but learning is still separated from optimization \citep{demirovic2019predict+, demirovic2019investigation}.

\subsubsection{Optimizing worker--task assignments.}

Given task predictions $\mathbf{\hat{Y}}$, we can optimize the task--worker assignments. Although various general algorithms have been proposed to solve assignment problems, our formulation can be solved analytically. Here, we present this solution and prove its optimality.

\begin{theorem}
$\mathbf{W} = \{w_i\}_{i=1}^N$ is a sequence of monotonically decreasing worker success rates such that $w_1 \geq \dots \geq w_N$ with $w_i \in [0,1]$ for $i = 1, \dots, N$. $\mathbf{\hat{R}} = \begin{pmatrix} \hat{r}_1 & \dots & \hat{r}_N \end{pmatrix}$ are the predicted task rewards arranged in decreasing order such that $\hat{r}_1 \geq \dotso \geq \hat{r}_N$. For the resulting expected profit matrix $P = \begin{pmatrix} p_{ij} \end{pmatrix}$ with $p_{ij} = w_i \hat{r}_j$, the optimal assignment is $A^* = I_N$. 
\end{theorem}


\begin{proof}

$A^* = I_N$ is a feasible solution: it is straightforward to verify that the identity matrix satisfies constraints \ref{eq:cond_task}, \ref{eq:cond_worker} and \ref{eq:integer} of the assignment problem. 
Moreover, the solution is the result of a greedy strategy: at each step $m$, we assign worker $w$ with probability $w_m$ to the highest remaining task $m$ with payoff $\hat{r}_m$. To prove the optimality of this strategy, we show that it does not deviate from the optimal solution at each step up until the final solution is obtained.

First, the best single worker--task assignment is selected: the highest profit $p_{ij}$ is $p_{11} = w_1 \hat{r}_1$; no other higher profit exists as no higher $w_i$ or $\hat{r}_j$ exist. Next, we continue this strategy of selecting the best remaining worker--task assignment until there are no tasks left. We can show that, at each step, no other assignment matrix leads to a larger profit than this one. At step $m$, the profit obtained given assignment matrix $A^*$ equals $p_{11} + p_{22} + \dotso + p_{mm} = w_1 \hat{r}_1 + w_2 \hat{r}_2 + \dotso + w_m \hat{r}_m$.

Deviating from $A^*$ at a certain step means that at least one worker must be assigned to another task. We prove that no alternative assignment leads to a higher profit. Consider switching the assignments of tasks $i$ and $j$ with $i < j$. In the case that task $j$ has already been assigned to a worker, we have:
{\parfillskip=0pt \emergencystretch=.5\textwidth \par}
\begin{alignat*}{4}
     &&p_{ii} + p_{jj} &\geq p_{ij} + p_{ji} &\\ \nonumber
\iff \quad &&w_i \hat{r}_i + w_j \hat{r}_j &\geq w_i \hat{r}_j + w_j \hat{r}_i &\\ \nonumber
\iff \quad &&w_i (\hat{r}_i - \hat{r}_j) &\geq w_j (\hat{r}_i - \hat{r}_j) &\\ \nonumber
\iff \quad &&w_i \geq w_j &\text{ and } \hat{r}_i - \hat{r}_j \geq 0. \nonumber
\end{alignat*}
In the case that task $j$ has not yet been assigned, we have:
\begin{alignat*}{4}
     &&p_{ii} &\geq p_{ij} &\\ \nonumber
\iff \quad &&w_i \hat{r}_i &\geq w_i \hat{r}_j &\\ \nonumber
\iff \quad &&w_i \geq 0 & \text{ and } \hat{r}_i \geq \hat{r}_j \nonumber
\end{alignat*}
In both cases, the final statements follow from $\mathbf{W}$ and $\mathbf{\hat{R}}$ being monotonically decreasing and $i<j$, or from $w_i \in \left[0, 1\right]$.
\end{proof}

\subsection{Integrated predict-and-optimize using learning to rank}

In this section, we present a novel integrated approach for solving the assignment problem in \cref{sec:problem formulation}. Previously, we showed how the optimal assignment is $A^* = I_N$ if $\mathbf{W}$ and $\mathbf{\hat{R}}$ are arranged in decreasing order. Given that $\mathbf{W}$ is defined as a decreasing sequence, the challenge of optimizing the assignment can also be seen as correctly predicting the order of expected task rewards $\mathbf{\hat{R}}$. This formulation is equivalent to an alternative interpretation of the assignment problem as finding the optimal assignments by permuting the rows and columns of the profit matrix $P$ such that the resulting sum of the elements on the diagonal is maximized, or formally \citep{krokhmal2009random}: 
\begin{equation}
    \min_{\pi \in \Pi_n} \sum_{i=1}^N p_{i, \pi(i)}
\end{equation}
for $\pi \in \Pi_N$ with $\Pi_N$ the set of all permutations of the indices $\{1, \dots , N\}$, i.e., $\pi : \{1, \dots, n\} \mapsto \{1, \dots, N\}$. In our case, we need to find the optimal permutation of available tasks $\pi(\mathbf{T})$.
{\parfillskip=0pt \emergencystretch=.5\textwidth \par}

In this formulation, the assignment problem can be seen as predicting the optimal permutation $\pi(\mathbf{T})$ based on characteristics of the available tasks. Formally, let $g_\theta: \mathcal{X} \to \mathbb{R}: \mathbf{x} \mapsto g_\theta(\mathbf{x})$ be a ranking model. The goal is to find parameters $\theta \in \Theta$ such that the ordering of the mapping of tasks $g_\theta(x_1) \geq \dotso \geq g_\theta(x_n)$ corresponds to the ordering of their rewards $r_1 \geq \dotso \geq r_N$. A ranking based on $g_\theta$ can be seen as a permutation $\pi$ of the indices $\{1, \dots , n\}$.

The expected profit of a permutation $\pi(\mathbf{T})$ given a capacity $W$ can be optimized directly using learning to rank. The key insight is that for a given permutation $\pi$ of tasks $\mathbf{T}$, the expected profit $\sum_{i=1}^N w_i \hat{r}_{\pi(i)}$ of a ranking is equivalent to its discounted cumulative gain (DCG), which is a commonly used class of metrics in learning to rank \citep{wang2013theoretical}. Typically, the DCG is defined with discount $\frac{1}{\text{log}_2(i+1)}$ and gain $2^{t_i} - 1$ for $i \in \{1, \dots, n\}$. However, to match the expected profit, our formulation uses discount $\{w_i\}^N_{i=1}$ corresponding to the capacity distribution, gain equal to $1$ for all $i$, and relevance $\hat{r}_i$. By dividing the DCG by its ideal value (IDCG), the normalized DCG (NDCG) is obtained: NDCG = $\frac{\text{DCG}}{\text{IDCG}}$ with NDCG $\in [0,1]$ \citep{murphy2012machine}.

Optimizing the NDCG (or equivalently, the expected profit) directly is challenging as it depends on the predicted relative positions of instances instead of the model's outputs $g_\theta(\mathbf{x_i})$. Nevertheless, various algorithms have been proposed for this task in the literature on learning to rank (e.g., \citep{valizadegan2009learning}). In this work, we use LambdaMART \citep{wu2008ranking, burges2010ranknet}, which uses a combination of the LambdaRank loss \citep{burges2006learning} and gradient boosting of decision trees \citep{friedman2001greedy} to construct the ranking model. LambdaMART is a widely used approach that achieved the best performance in the Yahoo! Learning To Rank Challenge \citep{burges2010ranknet, chapelle2011yahoo, li2014learning}. In this way, we can train a ranking model $g_\theta$ to optimize the NDCG or expected profit of the assignments directly.

Finally, we need to specify each task's relevance, which serves as a label according to which the ranking would ideally be constructed. Because the ranking corresponds to the priority that should be given to tasks, it should respect the ordering in terms of both outcome $y_i$ and task payoffs $\mathbf{v_i}$. In other words, successful tasks should be more relevant than unsuccessful tasks, and a more profitable task should be more relevant. Therefore, we use a task's reward $r_i$ as a cost-sensitive relevance, as it uses an instance's class label $y_i$ and its cost matrix $\mathbf{c}_i$ (see \cref{eq:cost_matrix}). By means of this approach, a positive task's relevance is the profit (or equivalently, the negative cost) obtained by classifying it positively minus the profit obtained by classifying it negatively; vice versa for negative tasks. Thus, we obtain the relevance or reward $r_i$ as follows:
\begin{equation*}
    r_i = y_i v^+_i + (1-y_i) v^-_i = y_i \left(c^\text{FN}_i -c^\text{TP}_i\right) + (1- y_i) \left(c^\text{TN}_i - c^\text{FP}_i\right)
    .
\end{equation*}
Alternatively, if the goal is to optimize for accuracy rather than cost, we can use class label $y_i$ as the relevance of instance $i$.

\section{Empirical results}

In this section, we empirically evaluate and compare the two-stage and the integrated approach for a variety of tasks. We use publicly available data from a variety of application areas. For each application, the goal is to optimally allocate resources  to optimize the expected cost given stochastic capacity.
All code for the experimental analysis will be made available online upon publication of this paper.

To compare the different approaches, we use gradient boosting to train the predictive models. Four different objectives are compared, depending on the task (classification or learning to rank) and on whether they aim to maximize precision or profit. xgboost denotes a conventional classification model using the cross-entropy loss $\mathcal{L}^\text{CE}$ (see \cref{eq:ce}), while csboost uses a cost-sensitive objective function $\mathcal{L}^\text{AEC}$ (see \cref{eq:aec}). LambdaMART uses the binary class label $y_i$, whereas csLambdaMART uses task payoffs $r_i$ as relevance. All models are implemented in Python using the \href{https://xgboost.readthedocs.io/en/latest/}{\texttt{xgboost}} package \citep{chen2015xgboost}. Gradient boosting is a popular methodology for both classification and ranking that has great predictive performance, as illustrated by recent benchmarking studies \citep{lessmann2015benchmarking, gunnarsson2021deep}.

\subsection{Data}

The data sets are enlisted in \cref{tab:data_overview} and stem from different application areas: customer churn prediction, credit scoring and direct marketing. They all concern binary classification where tasks are either successful or unsuccessful.
Resources are limited and stochastic: we assume a lognormal capacity distribution $W \sim \mathcal{LN}(\mu = \text{log}(100), \sigma = 1)$. 

The cost matrices are taken from earlier work on cost-sensitive classification (see \cref{tab:cost_matrices}).
In churn prediction, we have $c^\text{FP}_i$ and $c^\text{FN}_i$ as, respectively, 2 and 12 times the monthly amount $A_i$ for KTCC following \citep{petrides2021csensemble}; whereas we follow the cost matrix given with the data set for TSC \citep{bahnsen2015novel}.
For credit scoring, we calculate the instance-dependent costs $c^\text{FP}_i$ and $c^\text{FN}_i$ as a function of the loan amount $A_i$ following \citep{bahnsen2014examplelogistic}.
In direct marketing, a positive classification incurs a fixed cost $c_f = 1$, while missing a potential success incurs an instance-dependent cost equal to the expected interest given $A_i$, following \citep{bahnsen2015exampletree}.
Similarly, in fraud detection, a positive prediction leads to an investigation that entails a fixed cost $c_f$, and missing a fraudulent transaction leads to a cost equal to its amount $A_i$. We use $c_f = 10$, following \citep{hoppner2022instance}.

\bgroup
\begin{table}[t]
\setlength{\tabcolsep}{10pt}
    \centering
    \begin{tabular}{L{90pt}HL{40pt}R{45pt}HR{45pt}L{150pt}}
    \toprule
    \bf Application & \bf Name & \bf Abbr. & \bf $N$ & $\mathbb{E}(R)$ [\%] & \% Pos & \bf Reference \\ \midrule
    \multirow{2}{*}{Churn prediction} & Kaggle Telco Customer Churn & {KTCC} & 7,032 & 2.34 & 26.58 & \citep{kaggle2017telco} \\
     & TV Subscription Churn & {TSC} & 9,379 & 1.76 & 4.79 & \citep{bahnsen2015novel} \\ \midrule
    \multirow{7}{*}{Credit scoring} & Home Equity & {HMEQ} & 1,986 &  & 19.95 & \citep{baesens2016credit} \\
     & BeNe1 Credit Scoring & {BN1} & 3,123 &  & 33.33 & \citep{lessmann2015benchmarking} \\
     & BeNe2 Credit Scoring & {BN2} & 7,190 &  & 30.00 & \citep{lessmann2015benchmarking} \\
     & VUB Credit Scoring & {VCS} & 18,917 & 0.87 & 16.95 & \citep{petrides2020cost} \\
     & UK Credit Scoring & {UK} & 30,000 &  & 4.00 & \citep{lessmann2015benchmarking} \\
     & UCI Default of Credit Card Clients & {DCCC} & 30,000 & 0.55 & 22.12 & \citep{yeh2009comparisons} \\
     & Give Me Some Credit & {GMSC} & 112,915 &  & 6.74 & / \\ \midrule
    \multirow{2}{*}{Direct marketing} & UCI Bank Marketing & {UBM} & 45,211 & 0.36 & 11.70 & \citep{moro2014data} \\
     & KDD Cup 1998 & {KDD} & 191,779 & & 5.07 & / \\ \midrule
    \multirow{3}{*}{Fraud detection} & Kaggle Credit Card Fraud & {KCCF} & 282,982 &  & 0.16 & \citep{dal2015calibrating} \\
     & Kaggle IEEE Fraud Detection & {KIFD} & 590,540 &  & 3.50 & / \\
     & APATE Credit Card Fraud & {ACCF} & 3,639,323 &  & 0.65 & \citep{vanvlasselaer2015apate} \\
    \bottomrule \\
    \end{tabular}
    \caption{\textbf{Data sets overview.} \normalfont{For each data set, we present the application area, abbreviation, number of instances ($N$), class imbalance in terms of proportion of positive instances (\% Pos), and corresponding reference. 
    }}
    \label{tab:data_overview}
\end{table}
\egroup

\subsection{Results}

We present the results using various performance metrics to compare the different models. The main metric of interest is either the expected precision or the expected profit given the stochastic capacity distribution $W$, depending on whether accuracy or profit is the objective. Furthermore, we present several additional classification and ranking metrics to gain more insight into the differences between the methodologies. For each metric, we present the average over all data sets and test whether the best performance is significantly different from the others using a Friedman test on the rankings with Bonferroni--Dunn post hoc correction \citep{demvsar2006statistical, garcia2008extension, garcia2010advanced} (see \cref{tab:metrics_overview}). 

\bgroup
\renewcommand{\arraystretch}{1.2}
\begin{table}[t]
\centering
\tabcolsep=0.15cm
\begin{subtable}{0.24\textwidth}
    \centering
    \begin{tabular}{C{9pt}R{9pt}|C{27pt}C{27pt}}
    \toprule
                                                    &        & \multicolumn{2}{c}{$y_i$}                   \\
                                                    &        & 0                & 1                      \\
    \cmidrule{1-4}
    \multicolumn{1}{c}{\multirow{2}{*}{$\hat{y}_i$}}  & 0      & 0                & 12$A_i$                \\
    \textbf{}                                       & 1      & 2$A_i$           & 0                      \\
    \bottomrule
    \end{tabular}
    \subcaption{\footnotesize Churn prediction}
\label{tab:cost_matrix_churn_prediction}
\end{subtable}%
\begin{subtable}{0.24\textwidth}
    \centering
    \begin{tabular}{C{9pt}R{9pt}|C{27pt}C{27pt}}
    \toprule
                                                    &        & \multicolumn{2}{c}{$y_i$}                     \\
                                                    &        & 0                & 1                        \\
    \cmidrule{1-4}
    \multicolumn{1}{c}{\multirow{2}{*}{$\hat{y}_i$}}  & 0      & 0                & $c^\text{FN}_i$               \\
    \textbf{}                                       & 1      & $c^\text{FP}_i$       & 0                        \\
    \bottomrule
    \end{tabular}
    \subcaption{\footnotesize Credit scoring}
\label{tab:cost_matrix_credit_scoring}
\end{subtable}%
\begin{subtable}{0.24\textwidth}
    \centering
    \begin{tabular}{C{9pt}R{9pt}|C{27pt}C{27pt}}
    \toprule
                                                    &        & \multicolumn{2}{c}{$y_i$}                   \\
                                                    &        & 0                & 1                      \\
    \cmidrule{1-4}
    \multicolumn{1}{c}{\multirow{2}{*}{$\hat{y}_i$}}  & 0      & 0                & $A_i$/$Int_i$          \\
    \textbf{}                                       & 1      & $c_f$            & $c_f$                  \\
    \bottomrule
    \end{tabular}
    \subcaption{\footnotesize Direct marketing}
\label{tab:cost_matrix_direct_marketing}
\end{subtable}%
\begin{subtable}{0.24\textwidth}
    \centering
    \begin{tabular}{C{9pt}R{9pt}|C{27pt}C{27pt}}
    \toprule
                                                    &        & \multicolumn{2}{c}{$y_i$}                   \\
                                                    &        & 0                & 1                      \\
    \cmidrule{1-4}
    \multicolumn{1}{c}{\multirow{2}{*}{$\hat{y}_i$}}  & 0      & 0                & $A_i$          \\
    \textbf{}                                       & 1      & $c_f$            & $c_f$                  \\
    \bottomrule
    \end{tabular}
    \subcaption{\footnotesize Fraud detection}
\label{tab:cost_matrix_fraud_detection}
\end{subtable}%
\caption{
\textbf{Cost matrices for the different application areas.} \normalfont{For each application, we present the costs associated with the outcomes in terms of predicted ($\hat{y}$) and actual ($y$) labels. $A_i$, $c^\text{FN}_i$, $c^\text{FP}_i$ and $Int_i$ represent instance-dependent costs and $c_f$ is a fixed cost.}}
\label{tab:cost_matrices}
\end{table}
\egroup

\subsubsection{Expected precision and expected profit.} In terms of expected precision, LambdaMART is the best performing model. Two models optimize for accuracy: LambdaMART and xgboost. The ranking model, LambdaMART, outperforms the classification model, xgboost. In terms of expected profit, the cost-sensitive ranking model, csLambdaMart, performs best. Of the two models optimizing for accuracy, xgboost and LambdaMART, the ranking model again achieves better results, although this difference is not statistically significant. We compare the trade-off between profit and precision in \cref{fig:profit_vs_precision} by plotting the rankings for each data set. To get an idea of the densities for the different models, we estimate it using a Gaussian kernel and show it for probabilities greater than $0.5$. Although the densities overlap, the ranking models outperform their classifying counterparts in their respective category.

\begin{table*}[]
    \centering
    \resizebox{\textwidth}{!}{
    \begin{tabular}{L{70pt}|C{70pt}C{70pt}|C{70pt}C{70pt}C{70pt}}
    \toprule
    \rowcolor{gray!5}
     & & & & & \\
    \rowcolor{gray!5}
     & \multirow{-2}{*}{\textbf{\shortstack{Expected\\ precision}}} & \multirow{-2}{*}{\textbf{\shortstack{Expected\\ profit}}} & \multirow{-2}{*}{\textbf{\shortstack{Average\\ precision}}} & \multirow{-2}{*}{\textbf{\shortstack{Spearman\\ correlation}}} & \multirow{-2}{*}{\textbf{AUCPC}} \\\midrule
    
    \cellcolor{red!15} xgboost & 0.4956 {\scriptsize $\pm$ 0.28} & 0.2115 {\scriptsize $\pm$ 0.18} & \textbf{0.9423} {\scriptsize $\pm$ \textbf{0.05}} & $-$0.0382 {\scriptsize $\pm$ 0.11} & 0.5548 {\scriptsize $\pm$ 0.25}	\\
    \cellcolor{green!15} csboost & 0.5865 {\scriptsize $\pm$ 0.24} & \underline{0.2940 {\scriptsize $\pm$ 0.19}} & 0.9075 {\scriptsize $\pm$ 0.07} & \textcolor{white}{$+$}0.2258 {\scriptsize $\pm$ 0.27} & 0.5657 {\scriptsize $\pm$ 0.24}	\\
    \midrule
    \cellcolor{blue!15} LambdaMART & \textbf{0.6555} {\scriptsize $\pm$ \textbf{0.26}} & 0.2471 {\scriptsize $\pm$ 0.16} & \underline{0.9366 {\scriptsize $\pm$ 0.05}} & $-$0.0302 {\scriptsize $\pm$ 0.15} & 0.5363 {\scriptsize $\pm$ 0.22}	\\
    \cellcolor{cyan!15} csLambdaMART & 0.6089 {\scriptsize $\pm$ 0.25} & \textbf{0.3587} {\scriptsize $\pm$ \textbf{0.17}} & 0.9336 {\scriptsize $\pm$ 0.05} & \textcolor{white}{$+$}\textbf{0.3829} {\scriptsize $\pm$ \textbf{0.28}} & \textbf{0.5999} {\scriptsize $\pm$ \textbf{0.23}}	\\
    
    \bottomrule
    \end{tabular}
    }
    \caption{\textbf{Evaluation metrics overview.} {We present an overview of the evaluation metrics, showing the average and standard deviation over all data sets. The best result is denoted in \textbf{bold}. Results that are not significantly different from the best result are \underline{underlined} ($\alpha = 0.05$). This is based on a Friedman test on the rankings with Bonferroni--Dunn post hoc correction.
    For both expected precision and profit, the ranking models perform best in their respective category. For the classification metric, average precision, the cost-insensitive classification model, xgboost, performs best. Conversely, for the ranking metrics, namely, Spearman correlation and the area under the cumulative profit curve, the ranking models outperform their classifying counterparts.}}
    \label{tab:metrics_overview}
    \vspace{-10pt}
\end{table*}

\subsubsection{Average precision, Spearman's $\rho$ and AUCPC.} These metrics weight all instances in the ranking equally as opposed to the previous metrics that weighted instances depending on their probability of being processed given the capacity distribution. On the one hand, we consider a classification metric: given the high degree of class imbalance for some data sets, we use the average precision \citep{davis2006relationship}. On the other hand, we consider two ranking metrics: the area under the cumulative profit curve and Spearman's rank correlation coefficient $\rho$.

First, we assess the quality of the model's predictions with a standard classification metric: average precision (AP). This metric summarizes the precision-recall curve and looks at the trade-off between precision and recall at different thresholds. The cost-insensitive classification model, xgboost, performs best. This result is expected as it is a classification model that optimizes for accuracy. However, this conventional classification metric has only weak correlation with the expected precision, suggesting that it is not a good indicator of performance.

We also adopt two ranking metrics. First, we use Spearman's rank correlation coefficient to quantify the correlation between the ranking of the predictions and the ranking of the task payoffs. csLambdaMart is the best performing model, outperforming csboost. Moreover, both cost-insensitive models have a correlation of approximately 0. This is as expected, as these models do not take payoff into account in their optimization. Second, the cumulative profit curve plots the profit that is realized as a function of the number $k$ of first ranked instances, with $k \in [1,N]$. We compare the area under this curve with the area of a random ranking and one of the optimal ranking to obtain a value between 0 and 1. csLambdaMART performs best, though neither the difference with xgboost nor csboost is statistically significant.

These results indicate that metrics for evaluating the ranking, such as Spearman's $\rho$ or the AUCPC, are more suitable than classification metrics, such as the average precision, for evaluating a model's performance under limited capacity. These findings suggest that ranking as a solution more closely aligns with the problem of allocating limited resources to uncertain tasks than classification, which is also confirmed by the superior performance of ranking models compared to classification models in terms of expected precision and expected profit.

\begin{figure}[t]
    \centering
    \includegraphics[width=0.6\textwidth]{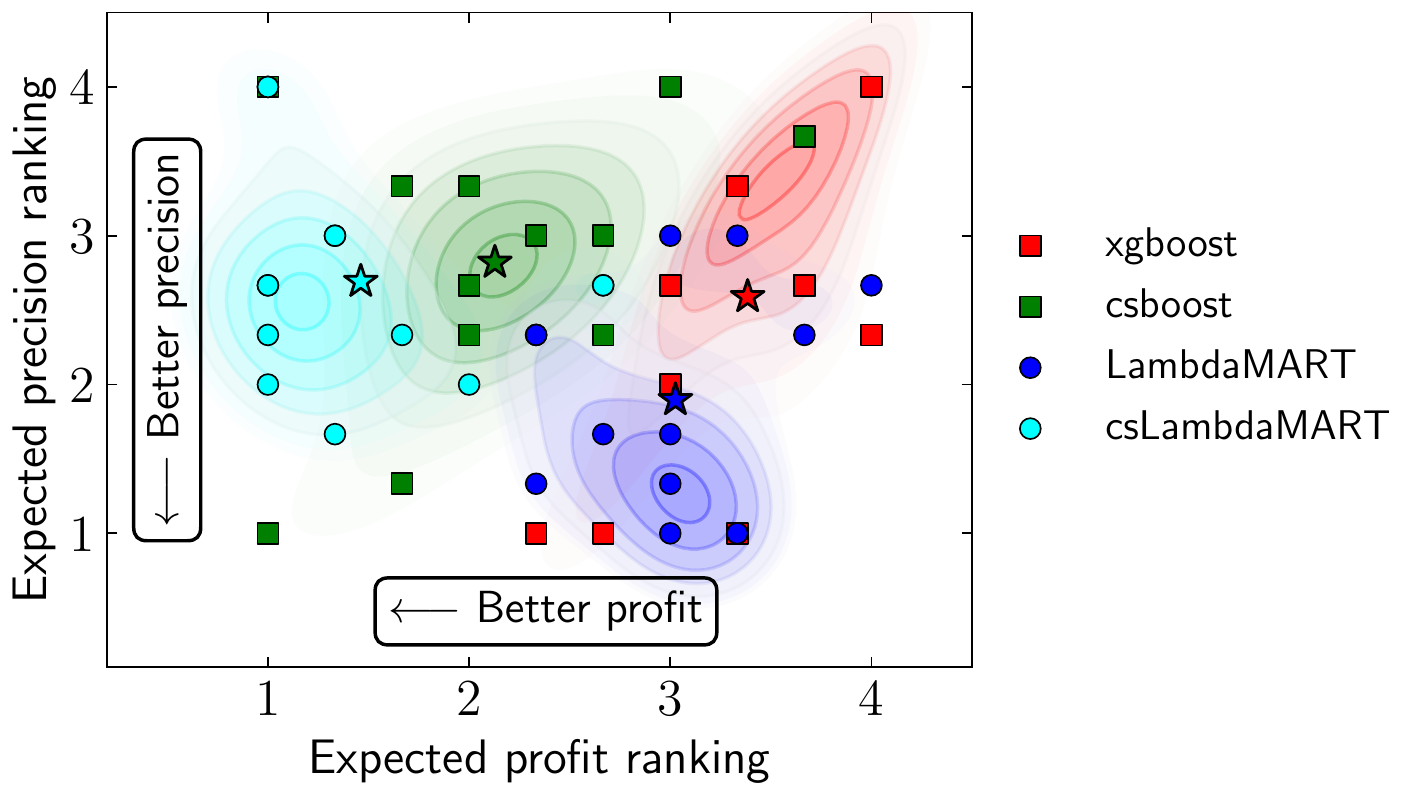}
    \caption{\textbf{Comparing the methodologies in terms of expected precision and profit.} \normalfont{We plot the rankings in terms of expected profit and expected precision for each method on each data set. Each method's average ranking is shown with a star ($\medwhitestar$). Moreover, the ranking density is fitted with a Gaussian kernel; for visual clarity, only probabilities greater than $0.5$ are shown. On average, csLambdaMART performs best in terms of expected profit, while LambdaMART performs best in terms of expected precision.
    }}
    \label{fig:profit_vs_precision}
\end{figure}

\subsubsection{Top $k$ metrics.}
Finally, we also consider metrics focusing solely on the top of the ranking. Given limited capacity, these are the instances that will be prioritized. We can evaluate this critical part of the ranking by looking at the precision and profit of the ranking for the first $k$ instances for different values of $k$ (see \cref{fig:top_k_precision_profit}). The ranking model optimizing for accuracy, LambdaMART, performs best in terms of precicision@$k$, while the ranking model optimizing for profit, csLambdaMART, has the best performance in terms of profit@$k$. 

\begin{figure}[t]
\centering
\begin{subfigure}{.3\textwidth}
  \centering
  \includegraphics[width=\textwidth]{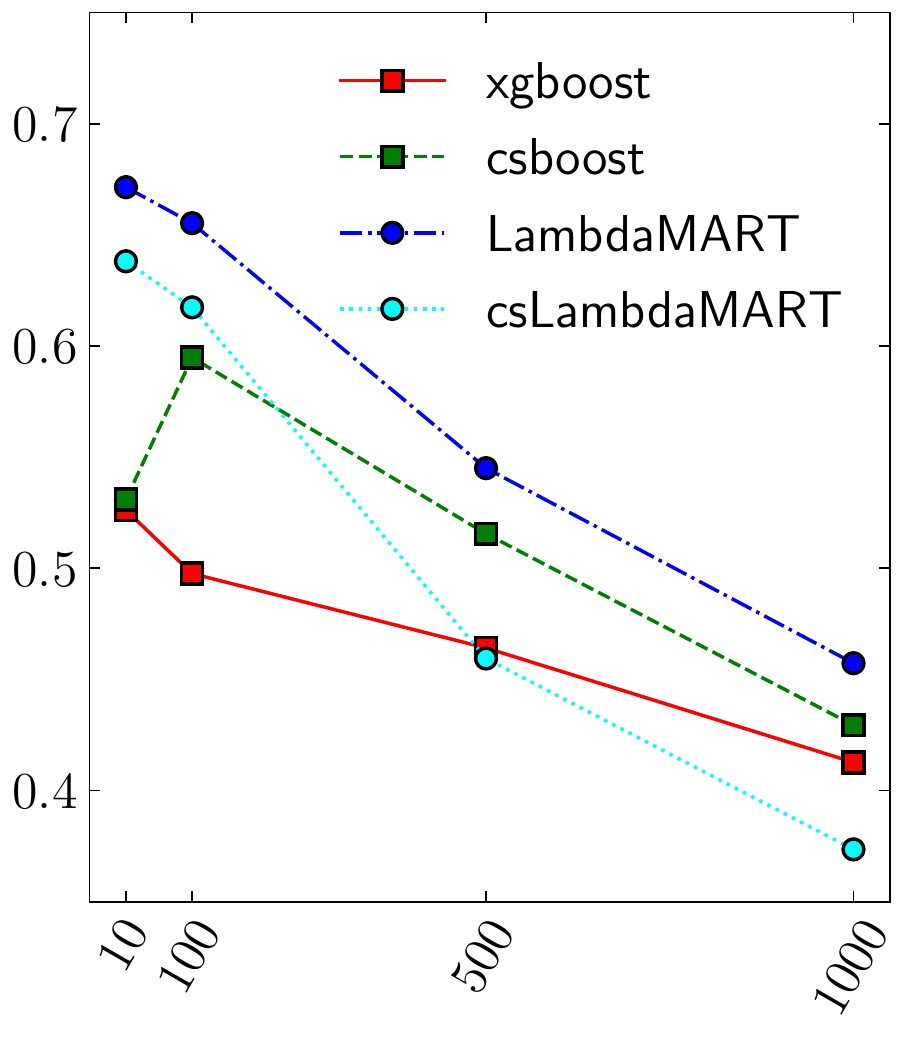}
  \caption{\footnotesize \textbf{Precision@$k$}}
  \label{fig:top_precision}
\end{subfigure}%
\begin{subfigure}{.3\textwidth}
  \centering
  \includegraphics[width=\textwidth]{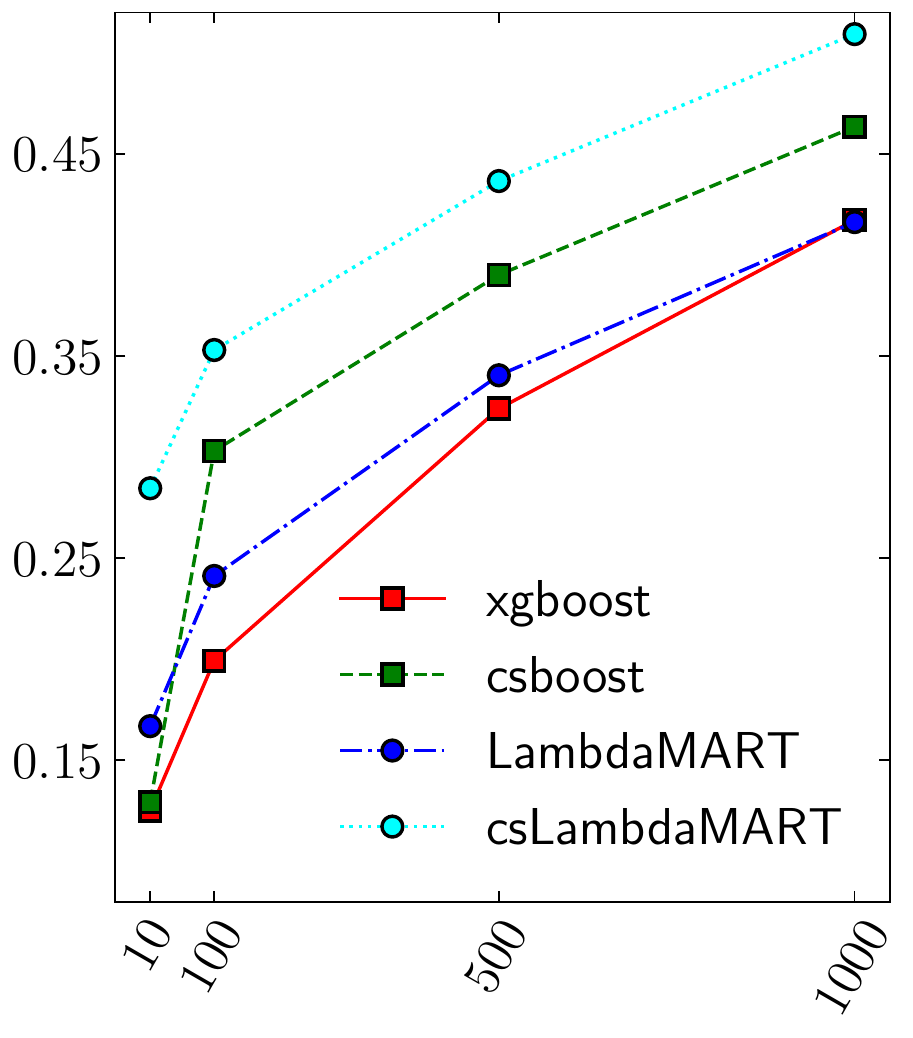}
  \caption{\footnotesize \textbf{Profit@$k$}}
  \label{fig:top_profit}
\end{subfigure}
\vspace{5pt}
\caption{\textbf{Evaluating the top $k$ ranked instances.} \normalfont{Precision \textbf{(a)} and profit \textbf{(b)} for the top $k$ instances in the ranking for the different models averaged over all data sets. The ranking models outperform the classifiers in the metric they optimize for: LambdaMART is the best in terms of precision; csLambdaMART has the best profit.}}
\label{fig:top_k_precision_profit}
\end{figure}

\section{Conclusion}

In this work, we formally introduced and defined a commonly encountered problem: how to optimally allocate limited, stochastic resource capacity to tasks with uncertain payoff to maximize the expected profit. Moreover, we contribute by proposing a novel integrated solution using learning to rank and empirically comparing it with a more conventional predict-then-optimize approach using a classification model.

Our findings illustrate the benefit of approaching this problem as a ranking problem, which allows us to consider the availability of limited and stochastic resources. Theoretically, we show how the expected profit for a given capacity distribution can be optimized directly using learning to rank with a specific formulation of the net discounted cumulative gain as the objective. Empirical results for a variety of applications show that ranking models achieve better performance in terms of expected profit or expected precision, depending on the objective. Moreover, good results in terms of ranking metrics are more indicative of good performance in terms of expected profit compared to conventional classification metrics. This illustrates how ranking is more closely aligned with the problem at hand compared to classifying. In summary, in the common scenario where decision-makers are constrained by limited resources, deciding upon resource allocation using classification models is inferior to using learning to rank. These findings have important implications for practitioners in a variety of application areas. 

Our work opens several promising directions for future research. For example, it would be interesting to consider a temporal variant of the assignment problem with tasks arriving sequentially in time. Although this problem has been studied extensively for stochastic or random arrival rates \citep{derman1972sequential, albright1972asymptotic, albright1974optimal}, future work could consider the addition of a predictive ranking model to address uncertainty regarding task outcomes. Another possible extension would be to consider tasks that require varying degrees of resources. For example, in credit scoring, loans with a large principal require more resources. Finally, a technical limitation of LambdaMART is the $O(N^2)$ complexity due to the pairwise calculation of the gradient. To address this issue, future work could look at approaches that calculate the gradient in a listwise fashion by considering the entire ranking simultaneously \citep{cao2007learning, xia2008listwise, ravikumar2011ndcg} with several recently proposed, efficient candidates \citep[e.g.,][]{sculley2009large, lucchese2017xdart, cakir2019deep}.

\bibliographystyle{unsrtnat}
\bibliography{references}

\end{document}